\documentclass[conference]{IEEEtran}
\IEEEoverridecommandlockouts
\usepackage{cite}
\usepackage{amsmath,amssymb,amsfonts}
\usepackage{graphicx}
\usepackage{textcomp}
\usepackage{lipsum}
\usepackage[usenames, dvipsnames]{xcolor}

\usepackage{algorithm}
\usepackage{algpseudocode}
\usepackage{hyperref}
\newcommand{\human}{\ensuremath{H}}
\newcommand{\aH}{\ensuremath{h}}
\newcommand{\aHt}{\ensuremath{\aH_t}}
\newcommand{\humant}{\ensuremath{\human_t}}
\newcommand{\robot}{\ensuremath{A}}
\newcommand{\aR}{\ensuremath{a}}
\newcommand{\aRt}{\ensuremath{\aR_t}}
\newcommand{\robott}{\ensuremath{\robot_t}}
\newcommand{\jointsystem}{\ensuremath{\robot \circ \human}}
\newcommand{\ent}{\ensuremath{\mathbf H}}
\newcommand{\info}{\ensuremath{\mathbf I}}

\newcommand{\M}{\mathcal M}

\newcommand{\paramspace}{\ensuremath{\Theta}}
\newcommand{\param}{\ensuremath{\mathbf{\theta}}}
\newcommand{\narms}{\ensuremath{N}}
\newcommand{\arm}{\ensuremath{a}}
\newcommand{\armt}{\ensuremath{\arm_t}}
\newcommand{\regret}{\ensuremath{\bar{R}}}
\newcommand{\regretT}{\ensuremath{\regret(t)}}
\newcommand{\prior}{\ensuremath{p}}
\newcommand{\mean}{\ensuremath{\mu}}
\newcommand{\meank}{\ensuremath{\mean_k}}

\newcommand{\meanstar}{\ensuremath{\mean^*}}
\newcommand{\paramk}{\ensuremath{\param_k}}
\newcommand{\pulls}{\ensuremath{T}}
\newcommand{\pullsk}{\ensuremath{\pulls_k}}
\newcommand{\pullskt}{\ensuremath{\pullsk(t)}}

\newcommand{\E}{\mathbb{E}}

\usepackage[utf8]{inputenc} 
\usepackage[T1]{fontenc}    
\usepackage{hyperref}       
\usepackage{url}            
\usepackage{booktabs}       
\usepackage{amsfonts}       
\usepackage{amsthm}
\usepackage{amssymb}
\usepackage{amsmath}
\usepackage{nicefrac}       
\usepackage{microtype}      
\usepackage{epigraph}
\usepackage{graphicx}
\graphicspath{ {images/}}
\usepackage{array}

\usepackage[normalem]{ulem}

\newtheorem{thm}{Theorem} 
\newtheorem{prop}[thm]{Proposition}

\newtheorem{cor}[thm]{Corollary}

\def\BibTeX{{\rm B\kern-.05em{\sc i\kern-.025em b}\kern-.08em
    T\kern-.1667em\lower.7ex\hbox{E}\kern-.125emX}}

\begin{document}

\title{The Assistive Multi-Armed Bandit
}

\author{
\IEEEauthorblockN{Lawrence Chan}
\IEEEauthorblockA{\small{\textit{{ University of California, Berkeley}}}\\
chanlaw@berkeley.edu}
\and
\IEEEauthorblockN{Dylan Hadfield-Menell}
\IEEEauthorblockA{\small{\textit{{ University of California, Berkeley}}}\\
dhm@berkeley.edu}
\and
\IEEEauthorblockN{Siddhartha Srinivasa}
\IEEEauthorblockA{\small{\textit{ University of Washington}}\\
siddh@cs.washington.edu}
\and
\IEEEauthorblockN{Anca Dragan}
\IEEEauthorblockA{\small{\textit{{ University of California, Berkeley}}}\\
anca@berkeley.edu}
}
\maketitle{}

\begin{abstract}
Learning preferences implicit in the choices humans make is a well studied problem in both economics and computer science. However, most work makes the assumption that humans are acting (noisily) optimally with respect to their preferences. Such approaches can fail when people are themselves learning about what they want. In this work, we introduce the assistive multi-armed bandit, where a robot assists a human playing a bandit task to maximize cumulative reward. In this problem, the human does not know the reward function but can learn it through the rewards received from arm pulls; the robot only observes which arms the human pulls but not the reward associated with each pull. We offer sufficient and necessary conditions for successfully assisting the human in this framework. Surprisingly, better human performance in isolation does not necessarily lead to better performance when assisted by the robot: a human policy can do better by effectively communicating its observed rewards to the robot. We conduct proof-of-concept experiments that support these results. We see this work as contributing towards a theory behind algorithms for human-robot interaction.
\end{abstract}


\section{Introduction}

\emph{Preference learning}~\cite{furnkranz2011preference} seeks to learn a predictive model of human preferences from their observed behavior. These models have been applied quite successfully in contexts like personalized news feeds~\cite{sakagami1997learning, li2010contextual, zhao2016user}, movie recommendations~\cite{basu1998recommendation, goel2009predicting, wang2018movie}, and human robot interaction~\cite{akgun2012keyframe, kuderer2012feature, fischer2016comparison, sadigh2016information, bajcsy2018learning}. We can learn this predictive model by fitting a utility function to revealed preferences~\cite{beigman2006learning, zadimoghaddam2012efficiently, balcan2014learning}, fitting parameters in a pre-specified human model~\cite{kingsley2010preference}, and applying contextual-bandit algorithms~\cite{li2010contextual}.

Central to all of these approaches is a fundamental assumption: human behavior is \textit{noisily-optimal} with respect to a set of \emph{stationary} preferences. Under this assumption, the problem can then be elegantly cast and analyzed as an \emph{inverse optimal control} (IOC)~\cite{kalman1964linear} or \emph{inverse reinforcement learning } (IRL) problem~\cite{russell1998learning}. Here, the human selects an action, takes it, and receives a reward,
which captures their internal preference (e.g. the enjoyment of having their desk organized a particular way). The robot only observes human actions and attempts to learn their preference under the assumption that the human likes the actions they selected; if you go for a particular desk configuration more frequently, the robot will assume that you like that configuration more. 

Unfortunately, in practice, this natural inference assumes stationarity, which is often violated. We have all experienced situations where our preferences change with experience and time~\cite{allais1979so, baron2000thinking}. This is particularly true in situations where we are \emph{ourselves learning} about our preferences as we are providing them~\cite{cyert1975adaptive,shogren2000preference}. For example, as we are organizing our desk, we might be experimenting with different configurations over time, to see what works. 

Now imagine a personal robot is trying to help you organize that same desk. If the robot believes you are optimal, it will infer the wrong preferences. Instead, if the robot accounts for the fact that you are learning about your preferences, it has a better shot at understanding what you want. Even more crucially, the robot can expose you to new configurations -- ones that might improve your posture, something you hadn't considered before and you were not going to explore if left to your own devices.  
\begin{figure}
    \centering
    \includegraphics[width=\columnwidth]{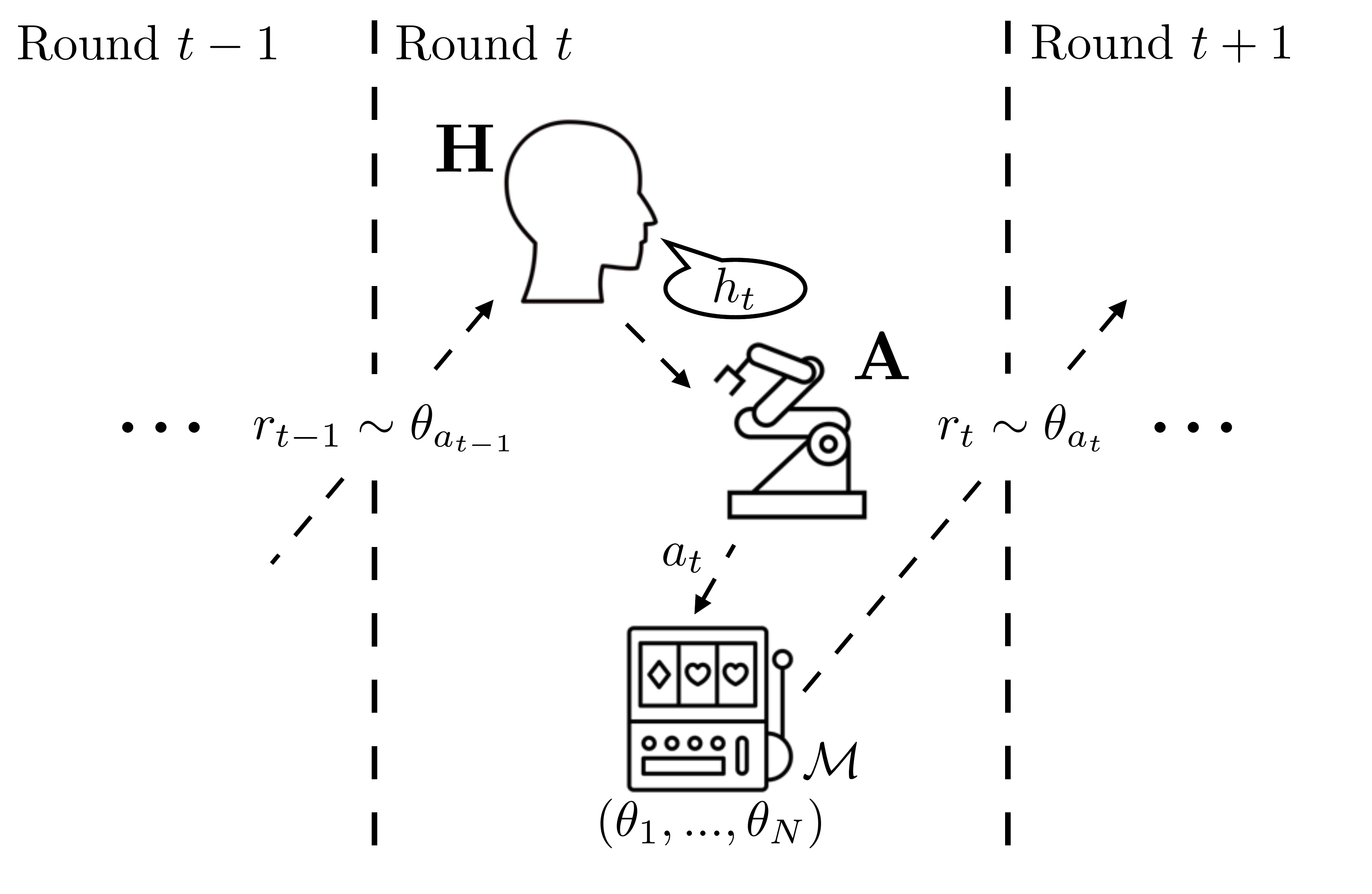}
    \caption{We introduce the assistive multi-armed bandit: a formalism for the problem of helping a learning agent optimize their reward. In each round, the human observes reward and tells the robot which arms they would like it to pull. The robot observes these requests, attempts to infer the reward values, and selects an arm to pull.}
    \label{fig:teaser}
\end{figure}




In this work, we formalize how a robot can actively assist humans who are themselves learning about their preferences.
 Our thesis is that by \emph{modeling} and \emph{influencing} the dynamics of human learning, the robot can enable the human-robot team to learn more effectively and outperform a human learning suboptimally in isolation. 

To this end, we introduce the \textit{assistive multi-armed bandit}: an extension of the classical \emph{multi-armed bandit} (MAB) model of learning.
In each round, the human selects an action, referred to in the bandit setting as an arm. 
However, the robot \emph{intercepts} their intended action and chooses a (potentially different) arm to pull. The human then observes the pulled arm and corresponding reward, and the process repeats (Figure~\ref{fig:teaser}).
We find this model surprisingly rich and fascinating. 
It captures the heart of collaboration: \emph{information asymmetry} and the cost of equalizing it. As the human learns about their preferences, they are compelled to communicate them to the robot as it decides their eventual reward. Analyzing this model allows us to understand the theoretical limits to assisting learning agents and the properties that make learners easier to assist.

Our contributions are the following: 1) we formalize the assistive multi-armed bandit; 2) we give weak sufficient conditions under which a human and robot team learns consistently, a lower bound on the cost of assuming noisy-optimality, and a mutual-information based upper bound on team performance; and 3) we use policy optimization~\cite{duan2016rl} to conduct an in-depth empirical validation of our theoretical results and investigate the effect of incorrectly modeling the human's learning strategy. We train \robot{} against a fixed learning strategy, e.g., $\epsilon$-greedy, and test it against a different learning strategy, e.g., Thompson sampling.  \emph{Our analysis shows a person that is better at learning does not necessarily lead to the human-robot team performing better - there are human learning strategies that are ineffective in isolation but communicate well and enable the robot to effectively assist.} In fact, human learning strategies that are \emph{inconsistent} in isolation, that is, failing a weak notion of asymptotic optimality, can allow the human-robot team to match \emph{optimal} performance in a standard multi-armed bandit.
Our results advance the theory behind algorithmic preference learning and provide guidance for structuring algorithms for human-robot interaction. 


\section{A Family of Bandits} 
\subsection{The Standard Multi-Armed Bandit}
A \textit{multi-armed bandit} (MAB) $\M$ is defined by:
\begin{itemize}
    \item \paramspace: a space of reward distributions parameters; $\param \in \paramspace$;
    \item \narms: an integer representing the number of arms;
    \item \prior: a distribution over $\Theta$.
\end{itemize}

At the start of the game, $\param$ is sampled from $\paramspace$ according to the prior \prior. At each timestep $t$, an arm $\armt \in [1,\ldots,\narms]$ is chosen. A reward $r_t \sim \param_{\armt}$ is sampled from the corresponding arm distribution. A \emph{strategy} is a mapping that determines the correct distribution to sample from given a history of reward observations and previous arm pulls: $K_t(\arm_1, r_1,  \ldots, \arm_{t-1}, r_{t-1})$. 

We use \meank{} to represent the mean of arm $k$, with parameters \paramk.  We use $j^*$ to represent the index of the best arm and \meanstar{} to represent its mean. \pullskt{} represents the number of pulls of arm $k$ up to and including time $t$. The goal of this game is to maximize the sum of rewards over time, or alternatively, to minimize the expectation of the regret \regretT{}, defined as:

\begin{equation}
    \regretT = \sum_t\left( \meanstar - \mean_{\armt} \right) = \sum_k\left( \meanstar - \meank \right)  \pullskt.
\end{equation}

\subsection{Stationary Inverse Optimal Control}
In preference learning, e.g., inverse reinforcement learning (IRL)\cite{ng2000algorithms,abbeel2004apprenticeship} and inverse optimal control (IOC)\cite{kalman1964linear}, an AI system observes (noisily-)optimal behavior and infers the reward function or preferences of that agent.
This relies on a key assumption that the agent being observed knows the value of actions it can take, at least in the sense that they are able to select optimal actions. In a multi-armed bandit setting, this set of assumptions corresponds to assuming that the human knows the parameters of the bandit, but has some small probability of picking a suboptimal arm. We refer to a human with this knowledge state and policy as implementing the $\epsilon$\emph{-optimal policy}. Inferring the reward of an $\epsilon$-optimal, or noisily-optimal, human can be thought of as solving a stationary IOC problem.



\subsection{The Inverse Multi-Armed Bandit}
Before formalizing the problem of \emph{assisting} a human who is learning, rather than noisily-optimal, we look at \emph{passively inferring} the reward from their actions. We call this the \textit{inverse bandit} problem. 
Each Inverse Bandit problem is defined by:

\begin{itemize}
    \item $\M$: a multi-armed bandit problem
    \item $\human$: a bandit strategy employed by the human, that maps histories of past actions and rewards to distributions over arm indices. $\humant: \aH_1 \times r_1 \times \cdots \times \aH_{t-1} \times r_{t-1} \rightarrow \Pi(N)$
\end{itemize}

The goal is to recover the reward parameter $\theta$ by observing \textit{only} the arm pulls of the human over time $\aH_1, ..., \aH_t$. 

Unlike the stationary IOC case, $\human$ does not have access to the true reward parameters. $\human$ receives the reward signal $r_t$ sampled according to $\theta$. As a consequence, the human arm pulls are not i.i.d.; the distribution of human arm pulls changes as they learn more about their preferences.

\subsection{The Assistive Multi-Armed Bandit}
In the \textit{assistive multi-armed bandit}, we have a joint system \jointsystem{} that aims to do well in an MAB $\M$. This strategy consists of two parts: the human player \human{} and robot player \robot{}. As in an MAB, the goal is to minimize the expected regret. The key difference between an assistive MAB and the standard MAB is that the policy is decomposed into a human component and a robot component. The goal is to capture scenarios where our goal, as designers of the robot \robott, is to optimize a reward signal which is only observed \textit{implicitly} through the actions of a human who is themselves learning about the reward function. 

The human and robot components of the policy are arranged in a setup similar to teleoperation. In each round:
\begin{enumerate}
    \item The human player \human{} selects an arm to suggest based on the history of previous arm pulls \emph{{and rewards}}: $\humant(\aR_1, r_1,\ldots, \aR_{t-1}, r_{t-1}) \in [1, \ldots, \narms]$.
    \item The robot player \robot{} selects which arm to actually execute based on the history of the human's attempts and the actual arms chosen:  $\robott(\aH_1, \aR_1, \ldots, \aH_{t-1}, \aR_{t-1}, \aHt) \in [1, \ldots, \narms]$.
    \item The human player \human{} observes the current round's arm and corresponding reward: $(\aRt, r_t \sim \param_{\robott})$.
\end{enumerate}

Unlike the inverse MAB or (stationary) IOC, the assistive MAB formalizes the problem of actually using learned preference knowledge to assist a human. Even if we are able to solve the inverse MAB, this is not useful if we can't actually help a learner reduce regret. We expect an optimal solution to the assistive MAB to improve on suboptimal learning, guide exploration, and correct for noise. 


\section{Theoretical Results}

\subsection{Hardness of Assistive MABs}
We consider the relative difficulty of assisting a person that knows what they want with assisting a person that is learning. We model the first situation as an assistive stationary IOC problem, and the second as an assistive MAB. First, we show that assistive stationary IOC is, as one might expect, quite easy in theory; we show that it is possible to infer the correct arm while making finitely many mistakes in expectation.



\begin{prop}
Suppose that $\human$'s arm pulls are i.i.d and let $f_i$ be the probability $\human$ pulls arm $i$. If $\human$ is noisily optimal, that is, $f_{j^*} >  f_i$ for all sub-optimal $i$, 
there exists a robot policy $\robot$ that has finite expected regret for every value of $\theta$:
\[\E[\regret(T)] \leq \sum_{i \not= j^*} \frac{\meanstar - \mean_{i}}{(\sqrt{f_{j^*}}-\sqrt{f_i})^2}\]
\begin{proof}
(Sketch) Our robot policy \robot{} simply pulls the most commonly pulled arm. 

Let $\hat f_i(t) = \frac{1}{t} \sum_{k=1}^t[\aHt=i]$ be the empirical frequency of \human's pulls of arm $i$ up to time $t$. Note that $\robott = i$ only if $\hat f_{j^*}(t) \leq \hat f_i(t)$. We apply a Chernoff bound to the random variable $\hat f_{j^*}(t) - \hat f_i(t)$. This gives that, for each $i$,
\begin{equation}
\label{eq:chernoffbound}
    \Pr(\hat f_i(t) \leq \hat f_j(t))\leq e^{-t(\sqrt{f_i}-\sqrt{f_j})^2}.
\end{equation}
 Summing Eq.~\ref{eq:chernoffbound} over $t$ and suboptimal arms gives the result.
\end{proof}
\end{prop}
\vspace{-2pt}

This is in contrast to the standard results about regret in an MAB: for a fixed, nontrivial MAB problem $\M$, any MAB policy has expected regret at least logarithmic in time on some choice of parameter $\theta$~\cite{lai1985asymptotically, mannor2004sample}:
\[\E[\regret(T)] \geq \Omega(log(T)).\] 
Several approaches based on Upper Confidence Bounds (UCB) have been shown to achieve this bound, implying that this bound is tight~\cite{lai1985asymptotically, agrawal1995sample, cappe2013kullback}. Nonetheless, this suggests that the problem of assisting a noisily-optimal human is significantly easier than solving a standard MAB.

The assistive MAB is at least as hard as a standard MAB. For the same sequence of arm pulls and observed rewards, the amount of information available to $\robot$ about the true reward parameters is upper bounded by the corresponding information available in a standard MAB. From a certain perspective, actually \emph{improving} on human performance in isolation is hopelessly difficult -- \robot{} does not get access to the reward signal, and somehow must still assist a person who does.

\subsection{Consistent Assisted Learning}
We begin with the simplest success criterion from the bandit literature: consistency. Informally, consistency is the property that the player eventually pulls suboptimal arms with probability 0. This can be stated formally as the average regret going to 0 in the limit: $\lim_{t\rightarrow \infty} {\regretT}/{t} = 0.$
In an MAB, achieving consistency is relatively straightforward: any policy that is greedy in the limit with infinite exploration (GLIE) is consistent~\cite{robbins1952some,Singh2000}. In contrast, in an assistive MAB, it is not obvious that the robot can implement such a policy when the $\human$ strategy is inconsistent. 
The robot observes no rewards and thus cannot estimate the best arm in hindsight. 

However, it turns out a weak condition on the human allows the robot-human joint system to guarantee consistency:

\begin{prop}
If the human \human{} implements a noisily greedy policy, that is, a policy that pulls the arm with highest sample mean strictly most often, then there exists a robot policy \robot{} such that \jointsystem{} is consistent. 
\label{thm:consistency}
\end{prop}

\begin{proof}
(Sketch) Fix a set of decaying disjoint exploration sequences $E_k$, one per arm, such that $\lim_{t\rightarrow \infty} \frac{1}{t}|E_k \cap \{1,...,t\}| \rightarrow 0$ and $\lim_{t\rightarrow \infty} |E_k \cap \{1,...,t\}| \rightarrow \infty$. In other words, each arm is pulled infinitely often, but at a decaying rate over time. 

Let $i_t$ be the arm most commonly pulled by $\human$ up until time $t$, and $\robot$ be defined by
\[\aRt = \left\{ \begin{array}{rl} k & t \in E_k \\ i_t& \textrm{otherwise} \end{array} \right. .\]
Note that this implies that for suboptimal k, $ \frac{1}{t}\pullskt \rightarrow 0$ in probability as $t\rightarrow \infty$, as the sample means of all the arms converge to the true means, and the rate of exploration decays to zero. This in turn implies that  $\jointsystem$ achieves consistency.
\end{proof}
\vspace{-2pt}
In other words, assistance is possible if the human picks the best actions in hindsight. 
This robot $\robot$ assists the human $\human$ in two ways. First, it helps the human explore their preferences -- $\jointsystem$ pulls every arm infinitely often. This fixes possible under-exploration in the human. Second, it stabilizes their actions and helps ensure that $\human$ does not take too many suboptimal actions - eventually, $\jointsystem$ converges to only pulling the best arm. This helps mitigate the effect of noise from the human.\\

\subsubsection{modeling learning as $\epsilon$-optimality leads to inconsistency}

We now investigate what occurs when mistakenly we model learning behavior as noisy-optimality. 

A simple way to make \jointsystem{} consistent when \human{} is noisily optimal is for \robot{} to pull the arm most frequently pulled by \human{}. 

\begin{prop}
If $\human{}$ plays a strategy that pulls the best arm most often and $\robot{}$ plays $\human$'s most frequently pulled arm, then $\jointsystem$ is consistent.
\end{prop}
\begin{proof} (Sketch) Eventually, $\human$'s most frequent arm converges to the best arm with probability 1 by hypothesis. At this point, $\robot$ will pull the best arm going forward and achieve a per-round regret of 0.
\end{proof}
\vspace{-2pt}
Next we consider the impact of applying this strategy when its assumptions are incorrect, i.e., $\human$ is learning. For simplicity, we assume $\human$ is greedy and pulls the best arm given the rewards so far. We will consider a $1\frac{1}{2}$-arm bandit: a bandit with two arms, where one has a known expected value and the other is unknown. We show that pairing this suboptimal-learner with the `most-frequent-arm' strategy leads the joint system $\jointsystem{}$ to be \emph{inconsistent}:

\begin{prop}
If $\human{}$ is a greedy learner and $\robot{}$ is `most-frequent-arm', then there exists an assistive MAB $\M$ such that $\jointsystem{}$ is inconsistent.
\label{thm:opt-inconsistent}
\end{prop}
\begin{proof} (Sketch) The proof consists of two steps. First, we show a variant of a classical bandit result: if $\human{}$ and $\robot{}$ output the constant arm in the same round, they will for the rest of time. Second, we show that this occurs with finite probability and get a positive lower bound on the per-round regret of $\jointsystem.$
\end{proof}
\vspace{-2pt}
\begin{figure}
    \centering
    \includegraphics[width=0.9\columnwidth]{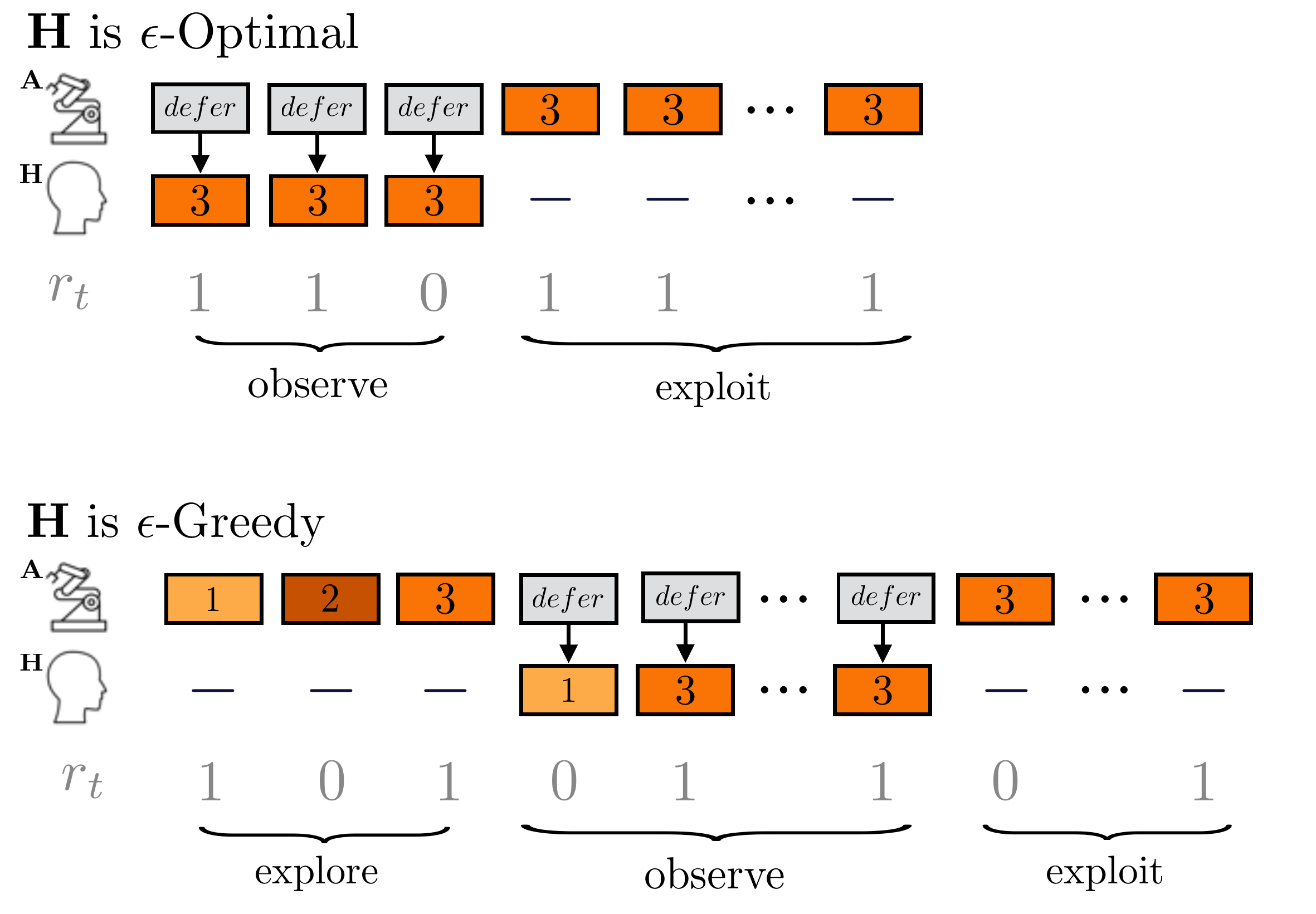}
    \caption{A comparison between assisting an $\epsilon$-optimal \human{} and an $\epsilon$-greedy \human{} in a modified assistive MAB (defined in Section~\ref{sec:other_modes}) where the robot \robot{} has to choose between acting and letting \human{} act. This creates a direct exploration-exploitation tradeoff that makes it easier to qualitatively analyze \robot{}'s behavior. At the top is whether \robot{} defers to the human or pulls an arm, followed by what \human{} pulls (if the robot defers), followed by the reward \human{} observes. When the robot models learning, the policy it learns has a qualitative divide into three components: explore, where the robot explores for the human; observe, when the robot lets the human pull arms; and exploit, when the robot exploits this information and pulls its estimate of the best arm. Crucially, the explore component is only found when learning is modeled. This illustrates Proposition \ref{thm:opt-inconsistent}, which argues that assisting an $\epsilon$-optimal \human{} is different from assisting a learning \human{}. 
    }
    \label{fig:traj_layout}
\end{figure}

While this is a simplified setting, this shows that the types of mistakes and suboptimality represented by learning systems \emph{are not} well modeled by the standard suboptimality assumptions used in research on recommendation systems, preference learning, and human-robot interaction. The suboptimality exhibited by learning systems is stateful and self-reinforcing. Figure~\ref{fig:traj_layout} shows the practical impact of modeling learning. It compares an optimal assistance policy for stationary IOC with an optimal policy for an assistive MAB. In general, assitive MAB policies seem to fit into three steps: \emph{explore} to give \human{} a good estimate of rewards; \emph{observe} \human{} to identify a good arm; and then \emph{exploit} that information.

MABs are the standard theoretical model of reinforcement learning and so this observation highlights the point that the term inverse reinforcement learning is somewhat of a misnomer (as opposed to inverse optimal control): IRL's assumptions about an agent (noisy optimality) lead to very different inferences than \emph{actually} assuming an agent is learning.
 
\subsection{Regret in Assistive Multi-Armed Bandits}
Having argued that we can achieve consistency for such a broad class of human policies in an assistive MAB, we now return to the question of achieving low regret. In particular, we investigate the conditions under which $\jointsystem{}$ achieve $O(\log(T))$ expected regret, as is possible in the standard MAB. 

For any given human $\human$, there exists a robot \robot{} such that \jointsystem{} does as well as $\human$: let $\robot$ copy the \human{}'s actions without modification; that is, $\aRt = \aHt$ for all $t$. So in the case where $\human$ achieves $O(\log(T))$ regret by itself, $\jointsystem{}$ can as well. 

However, a more interesting question is that of when we can successfully assist an suboptimal $\human$ that achieves $\omega(\log(T))$ regret. \emph{While one may hypothesize that better human policies lead to better performance when assisted, this is surprisingly not the case, as the next section demonstrates. }

\subsubsection{An inconsistent policy that is easy to assist}

\begin{figure}
    \centering
    \includegraphics[width=\columnwidth]{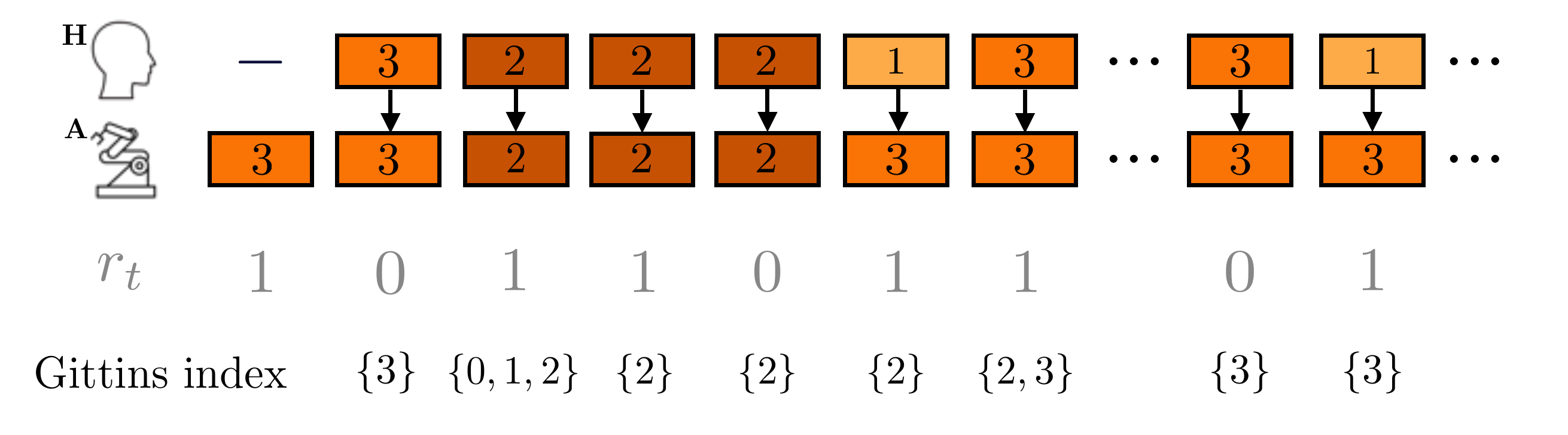}
    \caption{In Proposition~\ref{thm:wsls_optimality} we show that it is possible to match the regret from optimal learning in an standard MAB when assisting the `win-stay-lose-shift' (WSLS) policy. This is because WSLS perfectly communicates the observed rewards to \robot. Here we show an example trajectory from an approximately optimal policy assisting WSLS (computed with Algorithm~\ref{alg:pol-opt}). At the top is what \human{} suggests, followed by what \robot{} pulls, followed by the reward \human{} observes. For comparison, we show the arms selected by the near-optimal Gittins index policy for each belief state. This highlights the importance of communicative learning policies in an assistive MAB.}
    \label{fig:wsls_near_opt}
\end{figure}

Consider a Beta-Bernoulli assistive MAB where rewards are binary: $r_t \in \{0, 1\}$. A classic bandit strategy here is `win-stay-lose-shift' (WSLS)~\cite{robbins1985some} which, as the name suggests, sticks with the current arm if the most recent reward is one: 
\begin{equation}
    \aHt = \left\{ \begin{array}{rl} \aR_{t-1} & r_{t-1} = 1 \\ \mathbf{Unif}(\{k | k \neq \aR_{t-1} \})  & r_{t-1} = 0\end{array} \right. 
\end{equation}

This a simple strategy that performs somewhat well empirically -- although it is easy to see that it is not consistent in isolation, let alone capable of achieving $O(\log(T))$ regret. Indeed, it achieves $\Theta(T)$ regret, as it spends a fixed fraction of its time pulling suboptimal arms. However, if \human{} can implement this strategy, the combined system can implement an arbitrary MAB strategy from the standard MAB setting, including those that achieve logarithmic regret. In other words, the robot can successfully assist the human in efficiently balancing exploration and exploitation \emph{despite only having access to the reward parameter through an inconsistent human}.

\begin{prop}
Let $R^*$ be the optimal regret for a Beta-Bernoulli multi-armed bandit. If \human{} implements the WSLS strategy in the corresponding assistive MAB, then there exists a robot strategy \robot{} such that $\robot \circ \human$ achieves regret $R^*$.
\label{thm:wsls_optimality}
\end{prop}

\begin{proof} (Sketch) The pair $(\aR_{t-1}, \aHt)$ directly encodes the previous reward $r_{t-1}$. This means that \robott{} can be an arbitrary function of the history of arm pulls and rewards and so it can implement the MAB policy that achieves regret $R^*$. 
\end{proof}
\vspace{-2pt}
Figure~\ref{fig:wsls_near_opt} compares a rollout of this \jointsystem (we describe the approach in Section~\ref{sec:pol-opt}) with a rollout of a near-optimal policy for a standard MAB. 

\subsubsection{Communication upper bounds team performance}

The WSLS policy is not unique in that it allows $\jointsystem{}$ to obtain logarithmic regret. A less interesting, but similarly effective policy, is for the human to directly encode their reward observations into their actions; the human need not implement a sensible bandit policy. For example, the following purely communicative $\human$ also works for a Beta-Bernoulli bandit:
\begin{equation}
    \aHt = \left\{ \begin{array}{rl} 0 & r_{t-1} = 0 \\1  & r_{t-1} = 1 \end{array} \right.
\end{equation}

We can generalize the results regarding communicative policies using the notion of mutual information, which quantifies the amount of information obtained through observing the human arm pulls. 

Let $\info(X;Y)$ be the mutual information between $X$ and $Y$, $\ent(X)$ be the entropy of $X$, and $\ent(X|Y)$ be the entropy of $X$ given $Y$. 


\begin{prop}
Suppose that the probability the robot pulls a suboptimal arm at time $t$ is bounded above by some function $f(t)$, that is $P(\robott \not = j^*) \leq f(t)$. Then the mutual information $\info(j^*;\aH_1\times \cdots \times \aHt)$ between the human actions up to time t and the optimal arm must be at least $(1-f(t))\log \narms - 1$. 
\label{thm:MI_bound}
\end{prop}

\begin{proof}
We can consider the multi-armed bandit task as one of deducing the best arm from the human's actions. This allows us to apply Fano's inequality~\cite{fano2008fano} to $P(\robott \not = j^*)$, and using the fact that the entropy of a Bernoulli random variable is bounded above by $1$, we get
\begin{align*}
    P(\robott \not = j^*) \log(\narms-1) &\geq \ent(\hat j^*|j^*) - \log 2\\
    &= \ent(\hat j^*) - \info(\hat j^*;j^*) -1 \\
    &\geq \log \narms - \info(j^*;\aH_1\times \cdots \times \aHt) - 1. 
\end{align*}
Rearranging terms and using $P(\robott \not = j^*) \leq f(t)$, we get
\begin{align*}
    \info(j^*;\aH_1\times \cdots \times \aHt) &\geq \log \narms -  f(t) \log(\narms-1) - 1 \\
    &\geq (1-f(t))\log \narms - 1.
\end{align*}
\end{proof}
\vspace{-2pt}
Intuitively, since the probability of error is bounded by $f(t)$, in $(1 - f(t))$ cases the human actions conveyed enough information for A to successfully choose the best action out of $N$ options. This corresponds to $\log N$ bits, so there needs to be at least $(1 - f(t)) \log N$ bits of information in \human{}'s actions.

\begin{cor}
Suppose that the probability the robot pulls a suboptimal arm at time $t$ is bounded above by some function $f(t)$, that is $P(\robott \not = j^*) \leq f(t)$. Then the mutual information $\info(\aR_1\times r_1 \times \cdots \times \aR_{t-1}\times r_{t-1};\aH_1\times \cdots \times \aHt)$ between the human actions up to time t and the human observations must be at least $(1-f(t))\log \narms - 1$. 
\end{cor}
\begin{proof}
Since the best arm is independent of the human actions given the human observations, this follows immediately from the data processing inequality and proposition \ref{thm:MI_bound}.
\end{proof}
\vspace{-2pt}
 In order to achieve regret logarithmic in time, we must have that $P(K_t \not = j^*) \leq \frac{C}{t}$ for some $C>0$. Applying proposition \ref{thm:MI_bound} above implies that we must have
\begin{align*}
    \info(j^*;\aH_1\times \cdots \times \aHt) \geq (1-\frac{C}{t})\log \narms - 1
\end{align*}

Note that the term $\info(\hat j^*; \aH_1\times \cdots \times \aHt)$ depends on \textit{both} the human policy and the robot policy - no learning human policy can achieve this bound unless the human-robot system $\jointsystem$ samples each arm sufficiently often. \emph{As a consequence, simple strategies such as inferring the best arm at each timestep and pulling it, cannot achieve the $\Theta(\log T)$ lower bound on regret.}

\section{Algorithms for Assistive Multi-Armed Bandits}


The optimal response to a given human strategy can be computed by solving a partially observed Markov decision process (POMDP)~\cite{kaelbling1998planning}. The state is the reward parameters $\theta$ and $\human$'s internal state. The observations are the human arm pulls. In this framing, a variety of approaches can be used to compute policies or plans, e.g., online Monte-Carlo planning~\cite{silver2010monte, guez2012efficient} or point-based value iteration~\cite{pineau2003point}. 

In order to run experiments with large sample sizes, our primary design criterion was fast online performance. This lead us to use a direct policy optimization approach. The high per-action cost of Monte-Carlo planners makes them impractical for this problem. Further, explicitly tracking $\theta$ and $\human$'s internal state is strictly harder than solving the inverse MAB.  



\begin{algorithm}
\begin{algorithmic}
\State human policy $\human{}$
\State initialize parameterized policy $\pi(w; \cdot)$, policy parameters $w$
\For{$i \leq  nItrs$}
    \State $\xi\textrm{s}, r\textrm{s} \gets$ \textsc{Sample-Trajectories}($\pi(w; \cdot)$, $Size$, $T$)
    \State $\hat{\partial w} \gets$ \textsc{Policy-Gradient}($\xi$s, $\pi(w; \cdot)$) \Comment{~\cite{sutton2000policy}}
    \State $w \gets w + \hat{\partial w}$
\EndFor\\

\Procedure{Sample-Trajectories}{$\pi(w, \cdot)$, $Size$, $T$}
\State initialize empty array $\xi$s
\For{$i \leq$ $Size$}
    \State $\param \sim \paramspace$
    \For{$t \leq  T$}
        \State $\aHt \sim \humant(\aH_1, r_1,\ldots, \aR_{t-1}, r_{t-1})$
        \State $\aRt \sim \pi(w; \aH_1, \aR_1, \ldots, \aH_{t-1}, \aR_{t-1}, \aHt)$
        \State $r_t \sim \param_{\aRt}$
    \EndFor
    \State $\xi \gets \xi = [(\aH_1, \aR_1, r_1), ..., (\aH_T, \aR_T, r_T)]$
    \State $\xi$s $\gets \xi$s$+[\xi]$
\EndFor \\
\Return $\xi$s
\EndProcedure
\end{algorithmic}
\caption{Policy Optimization for the Assistive MAB}
\label{alg:pol-opt}
\end{algorithm}

Our approach applies the policy optimization algorithm of \cite{duan2016rl} to assistive MABs. Given an assistive MAB $(\M, \human)$, we sample a batch of reward parameters $\theta$ from the prior $p(\Theta)$; generate trajectories of the form $\xi = [(\aH_1, \aR_1, r_1), ..., (\aHt, \aRt, r_t)]$ from $\human$ and the current robot policy $\robot^{(i)}$; and use the trajectories to update the robot policy to $\robot^{(i+1)}$. 
During this offline training stage, since we are sampling reward parameters rather than using the ground truth reward parameters, we can use the generated rewards $r_t$ to improve on $\robot^{(i)}$. 

We represent \robot's policy as a recurrent neural network (RNN). At each timestep, it observes a tuple $(\aR_{t-1}, \aHt)$ where $\aR_{t-1}$ is the most recent robot action and $\aHt$ is the most recent human action. In response, it outputs a distribution over arm indices, from which an action is sampled. Given a batch of trajectories, we use an approximate policy gradient method~\cite{sutton2000policy} to update the weights of our RNN\footnote{Our code is available online at \url{https://github.com/chanlaw/assistive-bandits}}. We summarize this procedure in Algorithm~\ref{alg:pol-opt}. In our experiments, we used Proximal Policy Optimization (PPO)~\cite{schulman2017proximal}, due to its ease of implementation, good performance, and relative insensitivity to hyperparameters. 

\section{Experiments}
In our experiments, we used a horizon 50 Beta-Bernoulli bandit with four arms. Pulling the $i$th arm produces a reward of one with probability $\theta_i$ and zero with probability $1-\theta_i$: $\Theta = [0,1]^4$. We assume a uniform prior over $\Theta$: $\theta_i \sim \textrm{Beta}(1, 1)$.

We consider 5 classes of human policy:
\begin{itemize}
    \item \textbf{$\epsilon$-greedy}, a learning \human{} that chooses the best arm in hindsight with probability $1-\epsilon$ and a random arm with probability $\epsilon$.\footnote{We performed grid search to pick an $\epsilon$ based on empirical performance, and found that $\epsilon=0.1$ performed best.}
    \item \textbf{WSLS}, the \emph{win-stay-lose-shift} policy~\cite{robbins1985some} sticks with the arm pulled in the last round if it returned 1, and otherwise switches randomly to another arm.
    \item \textbf{TS}, the \emph{Thompson-sampling} policy~\cite{thompson1933likelihood} maintains a posterior over the arm parameters, and chooses each arm in proportion to the current probability it is optimal. This is implemented by sampling a particle from the posterior of each arm, then pulling the arm associated with the highest value.
    \item \textbf{UCL}, the \emph{upper-credible limit} policy \cite{reverdy2014modeling}  is an algorithm similar to Bayes UCB \cite{kaufmann2012bayesian} with softmax noise, used as a model of human behavior in a bandit environment.\footnote{We set $K=4$ and softmax temperature $\tau=4$.}
    \item \textbf{GI}, the \emph{Gittins index} policy~\cite{gittins1979bandit} is the Bayesian optimal solution to an infinite horizon discounted objective MAB.\footnote{We follow the approximations described by Chakravorty and Mahajan in~\cite{chakravorty2014multi}, and choose a discount rate ($\gamma=0.9$) that performs best empirically using grid search.}
\end{itemize}
In addition, we also defined the following noisily-optimal human policy to serve as a baseline:
\begin{itemize}
    \item \textbf{$\epsilon$-optimal}, a \emph{fully informed} \human{}  that knows the reward parameters $\theta$, chooses the optimal arm with probability $1-\epsilon$, and chooses a random action with probability $\epsilon$.\footnote{We set $\epsilon$ to match that of the $\epsilon$-greedy policy.} 
\end{itemize}


\subsection{Inverse Multi-Armed Bandit}
\label{sec:inverse_bandit}

\begin{table}[t]
\caption{Log-density of true reward params in a horizon 5 inverse MAB
}
\begin{center}
    \begin{tabular}{l|c|c|}
\cline{2-3}
 & \multicolumn{2}{c|}{Assumed \human{} Policy} \\ \hline
\multicolumn{1}{|l|}{Actual \human{} Policy} & \multicolumn{1}{l|}{\textbf{Correct Policy}} & \multicolumn{1}{l|}{\textbf{$\epsilon$-Optimal}} \\ \hline
\multicolumn{1}{|l|}{\textbf{$\epsilon$-greedy}} & 0.49 & -0.23 \\ \hline
\multicolumn{1}{|l|}{\textbf{WSLS}} & 0.95 & 0.13 \\ \hline
\multicolumn{1}{|l|}{\textbf{TS}} & 0.02 & -0.23 \\ \hline
\multicolumn{1}{|l|}{\textbf{UCL}} & 0.03  & -0.30 \\ \hline
\multicolumn{1}{|l|}{\textbf{GI}} & 0.94 & 0.20 \\ \hline\hline
\multicolumn{1}{|l|}{\textbf{$\epsilon$-Optimal}} & -- & 1.55 \\ \hline
\end{tabular}
\end{center}

\label{tbl:inverse_bandit_inference}
\end{table}

Our first experiment investigates the miscalibration that occurs when we do not model learning behavior. A robot that doesn't model human learning will be overconfident. To show this, we use the Metropolis-Hastings algorithm~\cite{metropolis1953equation} to approximate the posterior over reward parameters \param{} given human actions. We compare the posterior we get when we model learning with the posterior that assumes \human{} is $\epsilon$-optimal.

We compared the log-density of the true reward parameters in the posterior conditioned on 5 \human{} actions, under both models, when \human{} is actually learning. We report the results in Table~\ref{tbl:inverse_bandit_inference}. For every learning human policy, we find that the log-density of the true reward parameters is significantly higher when we model learning than when we do not. In the case of $\epsilon$-greedy and TS, we find that the posterior that fails to model learning assigns negative log-density to the true parameters. This means the posterior is a \emph{worse estimate} of $\theta$ than the prior.

\subsection{Assistance is Possible}
\label{sec:assist_possible}
\begin{figure}[b]
    \centering
    \includegraphics[width=.98\columnwidth]{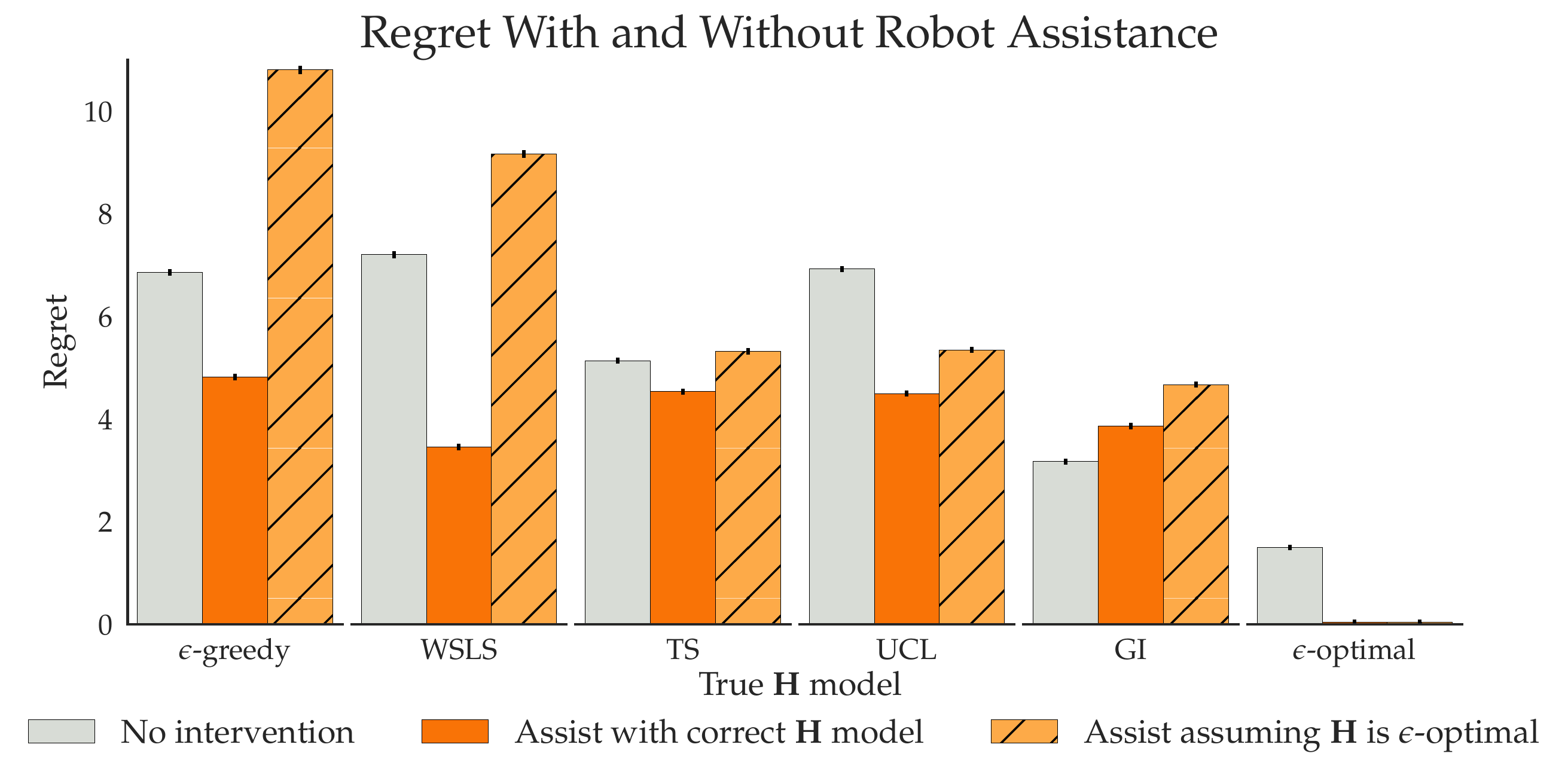}
    \caption{Averaged regret of various human policies (lower = better) over 100,000 trajectories when unassisted, assisted with the correct human model, and assisted assuming that the human is noisily-optimal. 
    Assistance lowers the regret of most learning policies, but it is important to model learning:
    ignoring that the human is learning can lead to worse performance than no assistance.
    Note that assisted WSLS performs almost as well as the Gittins Index policy, an empirical verification of Proposition~\ref{thm:wsls_optimality}. 
    }
    \label{fig:we-can-help}
\end{figure}
Propositions \ref{thm:consistency} and \ref{thm:wsls_optimality} prove that it is possible to assist suboptimal learners in theory. In this section, we show that it is possible in practice. We use Algorithm~\ref{alg:pol-opt} to train a recurrent policy for each human policy. In Fig.~\ref{fig:we-can-help}, we report the performance with assistance and without.

\subsubsection{Policy optimization details}
\label{sec:pol-opt}
To alleviate the problem of exploding and vanishing gradients~\cite{pascanu2013difficulty}, we use Gated Recurrent Units (GRU)~\cite{cho2014properties} as the cells of our recurrent neural network. The output of the GRU cell is fed into a softmax function, and this output is interpreted as the distribution over actions. To reduce to variance in our policy gradient estimate, we also use a value function baseline ~\cite{degris2012model} and apply Generalized Advantage Estimation (GAE)~\cite{schulman2015high}. We used weight normalization~\cite{salimans2016weight} to speed up training. We used a batch size of $250000$ timesteps or $5000$ trajectories per policy iteration, and performed $100$ policy iterations using PPO.

\begin{figure}
    \centering
    \includegraphics[width=0.93\columnwidth]{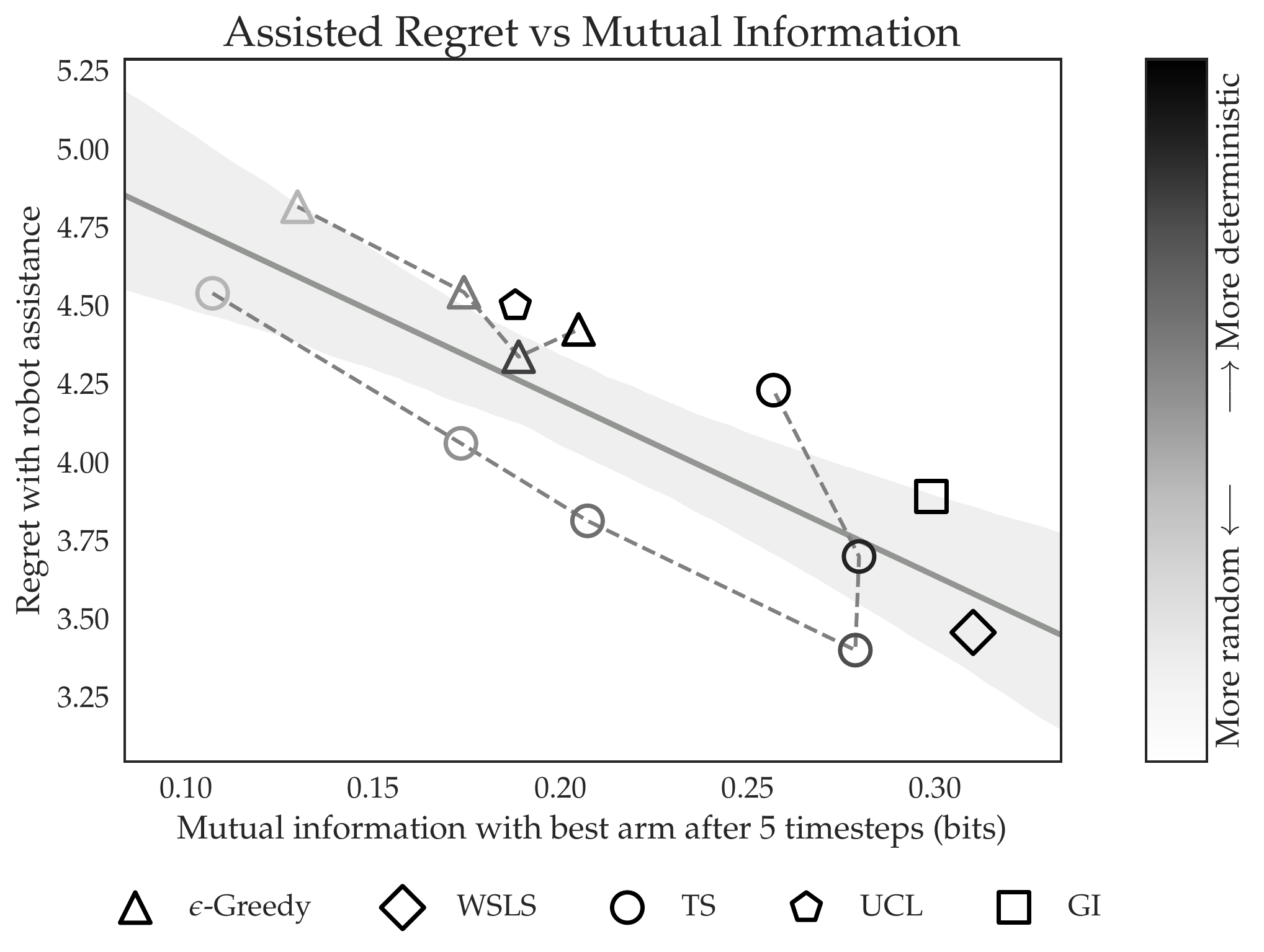}
    \caption{The assisted regret of various policies, plotted against the mutual information between the best arm and the policy's actions in the first 5 timesteps. We also plot the best-fit line, with 95\% confidence interval, for the regression between assisted regret and mutual information. We augmented our policies with variants of $\epsilon$-greedy and Thompson sampling with less randomness. Policies with high mutual information lead to lower regret when assisted, supporting our theoretical findings.}
    \label{fig:regrest_vs_mi}
\end{figure}

\subsubsection{Results}
To quantitatively test the hypotheses that our learned models successfully assist, we perform a two factor ANOVA. We found a significant interaction effect, $F(3, 999990) = 1778.8$, $p < .001$, and a post-hoc analysis with Tukey HSD corrections showed that we were able to successfully assist the human in all four sub-optimal learning policies ($p < .001$).

Our WSLS results agree with Proposition~\ref{thm:wsls_optimality}. Assisted WSLS achieves a regret of $3.5$, close to the regret of the best-performing unassisted policy, $3.2$. The gap in reward is due to our choice to employ approximate policy optimization. 
We provide an example trajectory in Fig.~\ref{fig:wsls_near_opt}. The actions selected are almost identical to those of the optimal policy. 

This suboptimality also accounts for the small increase in regret when assisting the Gittins index policy.


\subsubsection{Modeling learning matters} Proposition \ref{thm:opt-inconsistent} shows that modeling learning matters. We compared assistance assuming that \human{} is $\epsilon$-optimal with assistance with the correct (learning) model. We report the results in Fig.~\ref{fig:we-can-help}. We found that the regret with the wrong model is higher than no intervention in every case but UCL. Assisted WSLS with the wrong model has double the regret of assisted WSLS with the correct model.   Proposition~\ref{thm:opt-inconsistent} shows that in theory, ignoring learning leads to inconsistency. Fig.~\ref{fig:we-can-help} shows that this mistake leads to higher regret empirically.

\subsection{Mutual Information Predicts Performance}
Proposition~\ref{thm:MI_bound} implies that high mutual information is required for good team performance. To verify this, we computed the mutual information for a variety of combined policies after 5 timesteps. Fig.~\ref{fig:regrest_vs_mi} plots this against the regret of the combined system. We consider several variants of $\epsilon$-greedy and TS that are more or less deterministic. We consider $\epsilon \in [0, 0.02, 0.05, 0.1].$ To make TS more deterministic, we use the mean of a sample of $n$ particles to select arms. We consider $n \in [1, 2, 3, 10, 30, \infty]$. 

\begin{table}[t]
\begin{center}
\caption{Increase in Reward from Robot Assistance}
    \begin{tabular}{l|c|c|c|c|c||c|}
\cline{2-7}
\multicolumn{1}{c|}{} & \multicolumn{6}{c|}{Assumed \human{} Policy} \\ \hline
\multicolumn{1}{|l|}{Actual \human{}} & \multicolumn{1}{c|}{\textbf{$\epsilon$-greedy}} & \multicolumn{1}{c|}{\textbf{WSLS}} & \multicolumn{1}{c|}{\textbf{TS}} &
\multicolumn{1}{c|}{\textbf{UCL}} &
\multicolumn{1}{c||}{\textbf{GI}} & \multicolumn{1}{c|}{\textbf{$\epsilon$-optimal}} \\ \hline
\multicolumn{1}{|l|}{\textbf{$\epsilon$-greedy}} & \textcolor{ForestGreen}{\textbf{2.13}} & \textcolor{red}{-0.60}& \textcolor{red}{-2.18} & \textcolor{red}{-2.20}
&
\textcolor{red}{-0.11} & \textcolor{red}{-3.95} \\ \hline
\multicolumn{1}{|l|}{\textbf{WSLS}} & \textcolor{ForestGreen}{0.94} & \textcolor{ForestGreen}{\textbf{3.75}} & \textcolor{ForestGreen}{0.80} &
\textcolor{ForestGreen}{0.10}
& 
\textcolor{red}{-2.21} & \textcolor{red}{-1.97} \\ \hline
\multicolumn{1}{|l|}{\textbf{TS}} & \textcolor{ForestGreen}{0.33} & \textcolor{ForestGreen}{\textbf{0.66}} & \textcolor{ForestGreen}{0.60} & \textcolor{ForestGreen}{0.44} 
&
\textcolor{red}{-1.53} & \textcolor{red}{-0.19} \\ \hline
\multicolumn{1}{|l|}{\textbf{UCL}} & \textcolor{ForestGreen}{1.76} & \textcolor{red}{-1.19} & \textcolor{ForestGreen}{\textbf{2.51}} &
\textcolor{ForestGreen}{2.43}  & \textcolor{ForestGreen}{0.74}
& \textcolor{ForestGreen}{{1.28}}\\ \hline
\multicolumn{1}{|l|}{\textbf{GI}} & \textcolor{red}{-1.09} & \textcolor{red}{\textbf{-0.28}} & \textcolor{red}{-0.77} & 
 \textcolor{red}{-0.85} 
& 
\textcolor{red}{-0.71} & \textcolor{red}{-1.50} \\ \hline\hline
\multicolumn{1}{|l|}{\textbf{$\epsilon$-optimal}} & \textcolor{ForestGreen}{0.24} & \textcolor{ForestGreen}{1.17} & \textcolor{ForestGreen}{{1.24}} & 
\textcolor{ForestGreen}{{1.28}}
& \textcolor{red}{-3.09} & \textcolor{ForestGreen}{\textbf{1.46}} \\ \hline 
\end{tabular}
\label{tbl:sensitivity}
\end{center}
\end{table}

Across this data, higher mutual information is associated with lower assisted regret, $r(10)=-.82$, $p<.001$. Furthermore, by looking at the $\epsilon$-greedy and TS results as a sequence, we can observe a clear and distinct pattern. Policies that are more deterministic tend to be easier to help. This is supported by the results in Table \ref{tbl:inverse_bandit_inference}, which shows that it is easier to infer reward parameters for WSLS and GI (i.e., the two policies with the highest mutual information) than TS and $\epsilon$-greedy.

\subsection{Sensitivity to Model Misspecification}
\label{sec:model_misspec}

In the previous three experiments, we assumed knowledge of the correct learning policy. In this experiment, we consider the implications of incorrectly modeling learning. We took the policies we trained in Section~\ref{sec:assist_possible} and tested them with every human policy. We report the net change in reward in Table~\ref{tbl:sensitivity}. We colored cases where the robot \robot{} successfully assists the human \human{} green, and cases where it fails to assist red. We bolded the best performance in each row.

Modeling learning (even with the incorrect model) generally leads to lower regret than assuming $\epsilon$-optimal for every learning \human{} policy. However, when the robot has the wrong model of learning, it can fail to assist the human. For example, $\epsilon$-greedy is only successfully assisted when it is correctly modeled.  This argues that research into the learning strategies employed by people in practice is an important area for future research. 

An intriguing result is that assuming $\epsilon$-greedy \emph{does} successfully assist all of the suboptimal learning policies. This suggests that, although some learning policies must be well modeled, learning to assist some models can be transferred to other models in some cases. On the other hand, trying to assist GI leads to a policy that hurts performance across the board. In future work, we plan to identify classes of learners which can be assisted by the same robot policy. 



\subsection{Other Paradigms of Assistance}
\label{sec:other_modes}
\begin{figure}
    \centering
    \includegraphics[width=0.98\columnwidth]{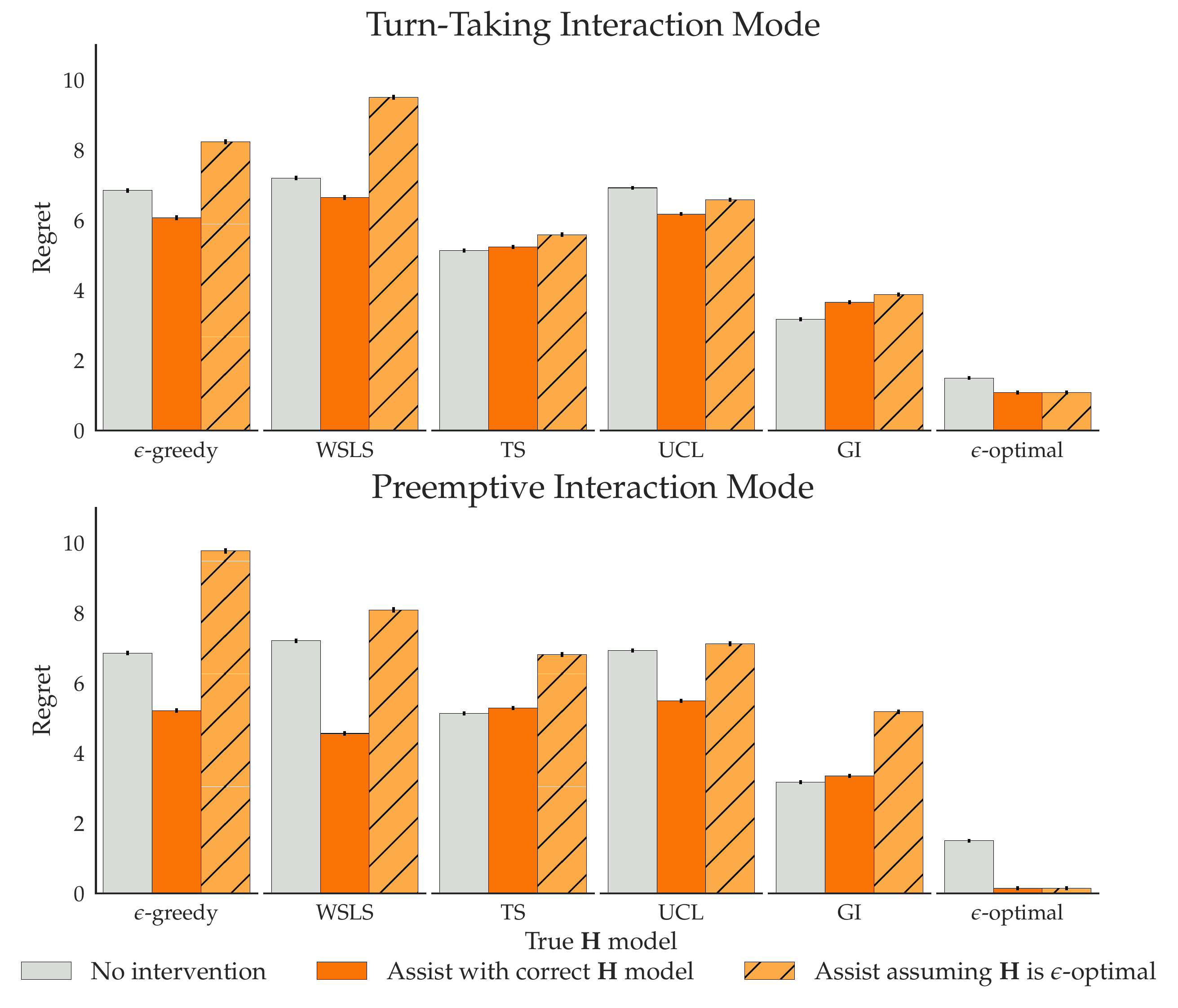}
    \caption{Averaged regret of various human policies (lower = better) over 100,000 trajectories under different interaction modes. Assistance lowers the regret of $\epsilon$-greedy, WSLS, and UCL in both the preemptive and turn-taking interaction modes. Assistance while ignoring learning is worse than no assistance in almost every case. This offers further support for the importance of modeling learning when assisting humans.}
    \label{fig:diff modes}
\end{figure}
The assistive multi-armed bandit considered so far only captures one mode of interaction. It is straightforward to consider extensions to different modes. We consider two such modes. The first is turn taking, where \human{} and \robot{} take turns selecting arms. This can be more difficult because the robot has to act in early rounds, when it has less information, and because the human has to act in later rounds, when \human{} may be noisy and the best arm has already been identified. 

The second variant we consider is preemptive interaction. In this case, \robot{} goes first and either pulls an arm or lets \human{} act. This creates an exploration-exploitation tradeoff. \robot{} only observes \human's arm pulls by actually allowing \human{} to pull arms and so it must choose between observing \human's behavior and exploiting that knowledge.   

Fig.~\ref{fig:diff modes} shows the experiment from Section~\ref{sec:assist_possible} applied to each of these interaction modes. The results are largely similar to those of teleoperation. We are able to assist the suboptimal policies and modeling learning as $\epsilon$-optimality increases regret in all cases. We see that WSLS is a less attractive policy in these settings: because \human{} actions are always executed when they are observed, it no longer makes sense for \human{} to employ a \emph{purely} communicative policy. However, we do 
still see results that confirm Proposition~\ref{thm:MI_bound}: more deterministic policies that reveal more information are easier to help. 

\subsubsection{Explore, observe, then exploit}
In looking at the policies learned for the preemptive interaction mode, we see an interesting pattern emerge. Because the policy has to choose between selecting arms directly and observing \human, by looking at rollouts of the learned policy we can determine when it is observing the human. We find that a clear pattern emerges. \robot{} initially \emph{explores} for \human: it selects arms uniformly to give \human{} a good estimate of $\theta$. Then, \robot{} \emph{observes} \human's arm pulls to identify the optimal arm. For the final rounds, \robot{} \emph{exploits} this information and pulls its estimate of the optimal arm. Fig.~\ref{fig:traj_layout} compares a representative trajectory with one that is optimized against an $\epsilon$-optimal \human.  

\section{Discussion}
\subsection{Summary} In this work, we studied the problem of assisting a human who is learning about their own preferences. Our central thesis is that by modeling and influencing the dynamics of a human's learning, we can create robots that can better assist people in achieving their preferences. We formalized this as the assistive multi-armed bandit problem, which extends the multi-armed bandit to account for teleoperation and human learning. We analyzed our formalism theoretically, then used policy optimization in proof-of-concept experiments that supported our theoretical results. Surprisingly, we found that a person that is better at learning in isolation does not necessarily lead to a human-robot team that performs better. We highlighted a theoretical connection between the amount of information communicated by the human policy and the best assisted performance, which we validated in our experiments. 



\subsection{Limitations and Future Work.} 
\subsubsection{Stateful environments} One significant limitation of this work is that we assume the environment the human is acting in is stateless. In practice, the environmental state changes over time, and the state can greatly influence the reward associated with certain actions.  This suggests natural extensions of the assistive multi-armed bandit to the contextual bandit~\cite{agarwal2014taming} and full Markov decision process (MDP)~\cite{bellman1957markov} settings. 

\subsubsection{Realistic human policies} Another significant limitation of this work is the use of simple bandit policies in place of actual human policies. In addition, we do not have access to the true human policy in any case. Future work can remedy this by incorporating more realistic policies, and can study to what extent these results generalize to assisting actual humans. 


\subsection{Closing Remarks}
The assistive multi-armed bandit is representative of a world where robots are supposed to assist people, even though people haven't figured out what they want yet. Laying down the theoretical foundations for these kinds of interaction paradigms is an important and under-served aspect of HRI.

\section*{Acknowledgements}
We thank the members of the InterACT lab and the Center for Human Compatible AI for helpful advice. This work was partially supported by OpenPhil, AFOSR, NSF, and NVIDIA.


\begin{thebibliography}{10}
\providecommand{\url}[1]{#1}
\csname url@samestyle\endcsname
\providecommand{\newblock}{\relax}
\providecommand{\bibinfo}[2]{#2}
\providecommand{\BIBentrySTDinterwordspacing}{\spaceskip=0pt\relax}
\providecommand{\BIBentryALTinterwordstretchfactor}{4}
\providecommand{\BIBentryALTinterwordspacing}{\spaceskip=\fontdimen2\font plus
\BIBentryALTinterwordstretchfactor\fontdimen3\font minus
  \fontdimen4\font\relax}
\providecommand{\BIBforeignlanguage}[2]{{%
\expandafter\ifx\csname l@#1\endcsname\relax
\typeout{** WARNING: IEEEtran.bst: No hyphenation pattern has been}%
\typeout{** loaded for the language `#1'. Using the pattern for}%
\typeout{** the default language instead.}%
\else
\language=\csname l@#1\endcsname
\fi
#2}}
\providecommand{\BIBdecl}{\relax}
\BIBdecl

\bibitem{furnkranz2011preference}
J.~F{\"u}rnkranz and E.~H{\"u}llermeier, ``Preference learning,'' in
  \emph{Encyclopedia of Machine Learning}.\hskip 1em plus 0.5em minus
  0.4em\relax Springer, 2011, pp. 789--795.

\bibitem{sakagami1997learning}
H.~Sakagami and T.~Kamba, ``Learning personal preferences on online newspaper
  articles from user behaviors,'' \emph{Computer Networks and ISDN Systems},
  vol.~29, no. 8-13, pp. 1447--1455, 1997.

\bibitem{li2010contextual}
L.~Li, W.~Chu, J.~Langford, and R.~E. Schapire, ``A contextual-bandit approach
  to personalized news article recommendation,'' in \emph{Proceedings of the
  19th international conference on World wide web}.\hskip 1em plus 0.5em minus
  0.4em\relax ACM, 2010, pp. 661--670.

\bibitem{zhao2016user}
Z.~Zhao, H.~Lu, D.~Cai, X.~He, and Y.~Zhuang, ``User preference learning for
  online social recommendation,'' \emph{IEEE Transactions on Knowledge and Data
  Engineering}, vol.~28, no.~9, pp. 2522--2534, 2016.

\bibitem{basu1998recommendation}
C.~Basu, H.~Hirsh, W.~Cohen \emph{et~al.}, ``Recommendation as classification:
  Using social and content-based information in recommendation,'' in
  \emph{Aaai/iaai}, 1998, pp. 714--720.

\bibitem{goel2009predicting}
D.~Goel and D.~Batra, ``Predicting user preference for movies using netflix
  database,'' \emph{Department of Electrical and Computer Engineering, Carniege
  Mellon University}, 2009.

\bibitem{wang2018movie}
H.~Wang and H.~Zhang, ``Movie genre preference prediction using machine
  learning for customer-based information,'' in \emph{Computing and
  Communication Workshop and Conference (CCWC), 2018 IEEE 8th Annual}.\hskip
  1em plus 0.5em minus 0.4em\relax IEEE, 2018, pp. 110--116.

\bibitem{akgun2012keyframe}
B.~Akgun, M.~Cakmak, K.~Jiang, and A.~L. Thomaz, ``Keyframe-based learning from
  demonstration,'' \emph{International Journal of Social Robotics}, vol.~4,
  no.~4, pp. 343--355, 2012.

\bibitem{kuderer2012feature}
M.~Kuderer, H.~Kretzschmar, C.~Sprunk, and W.~Burgard, ``Feature-based
  prediction of trajectories for socially compliant navigation.'' in
  \emph{Robotics: science and systems}, 2012.

\bibitem{fischer2016comparison}
K.~Fischer, F.~Kirstein, L.~C. Jensen, N.~Kr{\"u}ger, K.~Kukli{\'n}ski, T.~R.
  Savarimuthu \emph{et~al.}, ``A comparison of types of robot control for
  programming by demonstration,'' in \emph{The Eleventh ACM/IEEE International
  Conference on Human Robot Interaction}.\hskip 1em plus 0.5em minus
  0.4em\relax IEEE Press, 2016, pp. 213--220.

\bibitem{sadigh2016information}
D.~Sadigh, S.~S. Sastry, S.~A. Seshia, and A.~Dragan, ``Information gathering
  actions over human internal state,'' in \emph{Intelligent Robots and Systems
  (IROS), 2016 IEEE/RSJ International Conference on}.\hskip 1em plus 0.5em
  minus 0.4em\relax IEEE, 2016, pp. 66--73.

\bibitem{bajcsy2018learning}
A.~Bajcsy, D.~P. Losey, M.~K. O'Malley, and A.~D. Dragan, ``Learning from
  physical human corrections, one feature at a time,'' in \emph{Proceedings of
  the 2018 ACM/IEEE International Conference on Human-Robot Interaction}.\hskip
  1em plus 0.5em minus 0.4em\relax ACM, 2018, pp. 141--149.

\bibitem{beigman2006learning}
E.~Beigman and R.~Vohra, ``Learning from revealed preference,'' in
  \emph{Proceedings of the 7th ACM Conference on Electronic Commerce}.\hskip
  1em plus 0.5em minus 0.4em\relax ACM, 2006, pp. 36--42.

\bibitem{zadimoghaddam2012efficiently}
M.~Zadimoghaddam and A.~Roth, ``Efficiently learning from revealed
  preference,'' in \emph{International Workshop on Internet and Network
  Economics}.\hskip 1em plus 0.5em minus 0.4em\relax Springer, 2012, pp.
  114--127.

\bibitem{balcan2014learning}
M.-F. Balcan, A.~Daniely, R.~Mehta, R.~Urner, and V.~V. Vazirani, ``Learning
  economic parameters from revealed preferences,'' in \emph{International
  Conference on Web and Internet Economics}.\hskip 1em plus 0.5em minus
  0.4em\relax Springer, 2014, pp. 338--353.

\bibitem{kingsley2010preference}
D.~C. Kingsley and T.~C. Brown, ``Preference uncertainty, preference learning,
  and paired comparison experiments,'' \emph{Land Economics}, vol.~86, no.~3,
  pp. 530--544, 2010.

\bibitem{kalman1964linear}
R.~E. Kalman, ``When is a linear control system optimal?'' \emph{Journal of
  Basic Engineering}, vol.~86, no.~1, pp. 51--60, 1964.

\bibitem{russell1998learning}
S.~Russell, ``Learning agents for uncertain environments,'' in
  \emph{Proceedings of the eleventh annual conference on Computational learning
  theory}.\hskip 1em plus 0.5em minus 0.4em\relax ACM, 1998, pp. 101--103.

\bibitem{allais1979so}
M.~Allais, ``The so-called allais paradox and rational decisions under
  uncertainty,'' in \emph{Expected utility hypotheses and the Allais
  paradox}.\hskip 1em plus 0.5em minus 0.4em\relax Springer, 1979, pp.
  437--681.

\bibitem{baron2000thinking}
J.~Baron, \emph{Thinking and deciding}.\hskip 1em plus 0.5em minus 0.4em\relax
  Cambridge University Press, 2000.

\bibitem{cyert1975adaptive}
R.~M. Cyert and M.~H. DeGroot, ``Adaptive utility,'' in \emph{Adaptive Economic
  Models}.\hskip 1em plus 0.5em minus 0.4em\relax Elsevier, 1975, pp. 223--246.

\bibitem{shogren2000preference}
J.~F. Shogren, J.~A. List, and D.~J. Hayes, ``Preference learning in
  consecutive experimental auctions,'' \emph{American Journal of Agricultural
  Economics}, vol.~82, no.~4, pp. 1016--1021, 2000.

\bibitem{duan2016rl}
Y.~Duan, J.~Schulman, X.~Chen, P.~L. Bartlett, I.~Sutskever, and P.~Abbeel,
  ``Rl2: Fast reinforcement learning via slow reinforcement learning,''
  \emph{arXiv preprint arXiv:1611.02779}, 2016.

\bibitem{ng2000algorithms}
A.~Y. Ng and S.~J. Russell, ``Algorithms for inverse reinforcement learning.''
  in \emph{Icml}, 2000, pp. 663--670.

\bibitem{abbeel2004apprenticeship}
P.~Abbeel and A.~Y. Ng, ``Apprenticeship learning via inverse reinforcement
  learning,'' in \emph{Proceedings of the twenty-first international conference
  on Machine learning}.\hskip 1em plus 0.5em minus 0.4em\relax ACM, 2004, p.~1.

\bibitem{lai1985asymptotically}
T.~L. Lai and H.~Robbins, ``Asymptotically efficient adaptive allocation
  rules,'' \emph{Advances in applied mathematics}, vol.~6, no.~1, pp. 4--22,
  1985.

\bibitem{mannor2004sample}
S.~Mannor and J.~N. Tsitsiklis, ``The sample complexity of exploration in the
  multi-armed bandit problem,'' \emph{Journal of Machine Learning Research},
  vol.~5, no. Jun, pp. 623--648, 2004.

\bibitem{agrawal1995sample}
R.~Agrawal, ``Sample mean based index policies by o (log n) regret for the
  multi-armed bandit problem,'' \emph{Advances in Applied Probability},
  vol.~27, no.~4, pp. 1054--1078, 1995.

\bibitem{cappe2013kullback}
O.~Capp{\'e}, A.~Garivier, O.-A. Maillard, R.~Munos, G.~Stoltz \emph{et~al.},
  ``Kullback--leibler upper confidence bounds for optimal sequential
  allocation,'' \emph{The Annals of Statistics}, vol.~41, no.~3, pp.
  1516--1541, 2013.

\bibitem{robbins1952some}
H.~Robbins, ``Some aspects of the sequential design of experiments,'' in
  \emph{Bulletin of the American Mathematical Society.}, 1952.

\bibitem{Singh2000}
\BIBentryALTinterwordspacing
S.~Singh, T.~Jaakkola, M.~L. Littman, and C.~Szepesv{\'a}ri, ``Convergence
  results for single-step on-policy reinforcement-learning algorithms,''
  \emph{Machine Learning}, vol.~38, no.~3, pp. 287--308, Mar 2000. [Online].
  Available: \url{https://doi.org/10.1023/A:1007678930559}
\BIBentrySTDinterwordspacing

\bibitem{robbins1985some}
H.~Robbins, ``Some aspects of the sequential design of experiments,'' in
  \emph{Herbert Robbins Selected Papers}.\hskip 1em plus 0.5em minus
  0.4em\relax Springer, 1985, pp. 169--177.

\bibitem{fano2008fano}
R.~M. Fano, ``Fano inequality,'' \emph{Scholarpedia}, vol.~3, no.~10, p. 6648,
  2008.

\bibitem{kaelbling1998planning}
L.~P. Kaelbling, M.~L. Littman, and A.~R. Cassandra, ``Planning and acting in
  partially observable stochastic domains,'' \emph{Artificial intelligence},
  vol. 101, no. 1-2, pp. 99--134, 1998.

\bibitem{silver2010monte}
D.~Silver and J.~Veness, ``Monte-carlo planning in large pomdps,'' in
  \emph{Advances in neural information processing systems}, 2010, pp.
  2164--2172.

\bibitem{guez2012efficient}
A.~Guez, D.~Silver, and P.~Dayan, ``Efficient bayes-adaptive reinforcement
  learning using sample-based search,'' in \emph{Advances in Neural Information
  Processing Systems}, 2012, pp. 1025--1033.

\bibitem{pineau2003point}
J.~Pineau, G.~Gordon, S.~Thrun \emph{et~al.}, ``Point-based value iteration: An
  anytime algorithm for pomdps,'' in \emph{IJCAI}, vol.~3, 2003, pp.
  1025--1032.

\bibitem{sutton2000policy}
R.~S. Sutton, D.~A. McAllester, S.~P. Singh, and Y.~Mansour, ``Policy gradient
  methods for reinforcement learning with function approximation,'' in
  \emph{Advances in neural information processing systems}, 2000, pp.
  1057--1063.

\bibitem{schulman2017proximal}
J.~Schulman, F.~Wolski, P.~Dhariwal, A.~Radford, and O.~Klimov, ``Proximal
  policy optimization algorithms,'' \emph{arXiv preprint arXiv:1707.06347},
  2017.

\bibitem{thompson1933likelihood}
W.~R. Thompson, ``On the likelihood that one unknown probability exceeds
  another in view of the evidence of two samples,'' \emph{Biometrika}, vol.~25,
  no. 3/4, pp. 285--294, 1933.

\bibitem{reverdy2014modeling}
P.~B. Reverdy, V.~Srivastava, and N.~E. Leonard, ``Modeling human decision
  making in generalized gaussian multiarmed bandits,'' \emph{Proceedings of the
  IEEE}, vol. 102, no.~4, pp. 544--571, 2014.

\bibitem{kaufmann2012bayesian}
E.~Kaufmann, O.~Capp{\'e}, and A.~Garivier, ``On bayesian upper confidence
  bounds for bandit problems,'' in \emph{Artificial Intelligence and
  Statistics}, 2012, pp. 592--600.

\bibitem{gittins1979bandit}
J.~C. Gittins, ``Bandit processes and dynamic allocation indices,''
  \emph{Journal of the Royal Statistical Society. Series B (Methodological)},
  pp. 148--177, 1979.

\bibitem{chakravorty2014multi}
J.~Chakravorty and A.~Mahajan, ``Multi-armed bandits, gittins index, and its
  calculation,'' \emph{Methods and Applications of Statistics in Clinical
  Trials: Planning, Analysis, and Inferential Methods, Volume 2}, pp. 416--435,
  2014.

\bibitem{metropolis1953equation}
N.~Metropolis, A.~W. Rosenbluth, M.~N. Rosenbluth, A.~H. Teller, and E.~Teller,
  ``Equation of state calculations by fast computing machines,'' \emph{The
  journal of chemical physics}, vol.~21, no.~6, pp. 1087--1092, 1953.

\bibitem{pascanu2013difficulty}
R.~Pascanu, T.~Mikolov, and Y.~Bengio, ``On the difficulty of training
  recurrent neural networks,'' in \emph{International Conference on Machine
  Learning}, 2013, pp. 1310--1318.

\bibitem{cho2014properties}
K.~Cho, B.~Van~Merri{\"e}nboer, D.~Bahdanau, and Y.~Bengio, ``On the properties
  of neural machine translation: Encoder-decoder approaches,'' \emph{arXiv
  preprint arXiv:1409.1259}, 2014.

\bibitem{degris2012model}
T.~Degris, P.~M. Pilarski, and R.~S. Sutton, ``Model-free reinforcement
  learning with continuous action in practice,'' in \emph{American Control
  Conference (ACC), 2012}.\hskip 1em plus 0.5em minus 0.4em\relax IEEE, 2012,
  pp. 2177--2182.

\bibitem{schulman2015high}
J.~Schulman, P.~Moritz, S.~Levine, M.~Jordan, and P.~Abbeel, ``High-dimensional
  continuous control using generalized advantage estimation,'' \emph{arXiv
  preprint arXiv:1506.02438}, 2015.

\bibitem{salimans2016weight}
T.~Salimans and D.~P. Kingma, ``Weight normalization: A simple
  reparameterization to accelerate training of deep neural networks,'' in
  \emph{Advances in Neural Information Processing Systems}, 2016, pp. 901--909.

\bibitem{agarwal2014taming}
A.~Agarwal, D.~Hsu, S.~Kale, J.~Langford, L.~Li, and R.~Schapire, ``Taming the
  monster: A fast and simple algorithm for contextual bandits,'' in
  \emph{International Conference on Machine Learning}, 2014, pp. 1638--1646.

\bibitem{bellman1957markov}
\BIBentryALTinterwordspacing
R.~Bellman, ``A markovian decision process,'' \emph{Journal of Mathematics and
  Mechanics}, vol.~6, no.~5, pp. 679--684, 1957. [Online]. Available:
  \url{http://www.jstor.org/stable/24900506}
\BIBentrySTDinterwordspacing

\end{thebibliography}
\end{document}